\documentclass[10pt,twocolumn,letterpaper]{article}

\usepackage{amsmath,amssymb,amsthm,bm}
\usepackage{graphicx}
\usepackage{subfigure}
\usepackage{booktabs}
\usepackage{multirow}
\usepackage{algorithm}
\usepackage{algorithmic}
\usepackage[table,xcdraw]{xcolor}
\usepackage{microtype}
\usepackage{caption}

\newtheorem{theorem}{Theorem}
\newtheorem{definition}{Definition}

\newtheorem{lemma}{Lemma}
\newtheorem{corollary}{Corollary}
\DeclareMathOperator*{\argmin}{arg\,min}

\newcommand{\ie}{i.e.}

\newcommand{\ourmethod}{MINO}

\usepackage{iccv} 
\definecolor{iccvblue}{rgb}{0.21,0.49,0.74}
\usepackage[pagebackref,breaklinks,colorlinks,allcolors=iccvblue]{hyperref}

\def\paperID{2725}
\def\confName{ICCV}
\def\confYear{2025}

\title{Is Meta-Learning Out? Rethinking Unsupervised Few-Shot Classification with Limited Entropy}

\author{
Yunchuan Guan$^{1,2}$,
Yu Liu$^{1}$\thanks{Co-corresponding authors. liu\_yu@hust.edu.cn, zhke@hust.edu.cn, zqshen@ntu.edu.sg,lenny.lilei.cs@gmail.com}, 
Ke Zhou$^{1}$\footnotemark[1],
Zhiqi Shen$^{2}$\footnotemark[1],
Jenq-Neng Hwang$^{3}$,
Serge Belongie$^{4}$,
Lei Li$^{3,4}$\footnotemark[1]\\
$^{1}$Huazhong University of Science and Technology \quad 
$^{2}$Nanyang Technological University\\
$^{3}$University of Washington \quad 
$^{4}$University of Copenhagen
}

\begin{document}
\maketitle

\begin{abstract}
Meta-learning is a powerful paradigm for tackling few-shot tasks. However, recent studies indicate that models trained with the whole-class training strategy can achieve comparable performance to those trained with meta-learning in few-shot classification tasks. To demonstrate the value of meta-learning, we establish an entropy-limited supervised setting for fair comparisons. Through both theoretical analysis and experimental validation, we establish that meta-learning has a tighter generalization bound compared to whole-class training. We unravel that meta-learning is more efficient with limited entropy and is more robust to label noise and heterogeneous tasks, making it well-suited for unsupervised tasks. Based on these insights, We propose MINO, a meta-learning framework designed to enhance unsupervised performance. MINO utilizes the adaptive clustering algorithm DBSCAN with a dynamic head for unsupervised task construction and a stability-based meta-scaler for robustness against label noise. Extensive experiments confirm its effectiveness in multiple unsupervised few-shot and zero-shot tasks. 

\end{abstract}

\section{Introduction}\label{sec:intro}
Meta-learning has emerged as a powerful paradigm for learning to adapt to unseen tasks~\cite{survey}, demonstrating better generalization ability in few-shot learning and reinforcement learning tasks~\cite{meta_reinforcement}. However, recent studies empirically demonstrate that a well trained embedding exhibits comparable or even better accuracy on several few-shot classification tasks~\cite{revisiting_few_shot, a_good_embedding,meta_baseline,li2024cpseg,few_shot_hm,few_shot_survey,li2025addressing}. These methods employ a vanilla strategy, \ie, whole-class training (WCT), where each sample is trained within a single-task learning framework. The empirical results suggest that the sophisticated bi-level optimization and the task organization paradigm of meta-learning are somewhat redundant.

Previous studies attempt to derive a better generalization error bound~\cite{generalization_PAC_3,genralization_information_1,generalization_stable_1} or convergence rate~\cite{on_converage,multi_converage} to improve meta-learning itself. However, little is known about why its theoretical superiority conflicts with the aforementioned experimental results. Meanwhile, a theoretical framework for comparing the meta-learning and the whole-class training has yet to be explored. Figure~\ref{fig:two_secnario} shows the unfair comparison between meta-learning and WCT. It suggests that although the same dataset is used, WCT requires distinguishing more categories than meta-training (MT) and consumes more annotation resources. Since the annotation process can be considered as an entropy-reduced process, we introduce the \textbf{entropy-limited supervised setting}.

Lemma~\ref{lemma:correct_num} suggests that the entropy-limited supervised setting represents an intermediate state. Depending on the available entropy, it can degrade into either the conventional supervised or unsupervised setting. Under this setting, we focus on meta-learning algorithms based on bi-level optimization and discuss three key issues: (1) establishing a theoretical framework for comparing meta-learning and WCT; (2) explaining why and when meta-learning outperforms WCT; (3) providing insights that enhance the performance of meta-learning.

\begin{figure}[t]
   \centering  
    \includegraphics[width=0.85\linewidth]{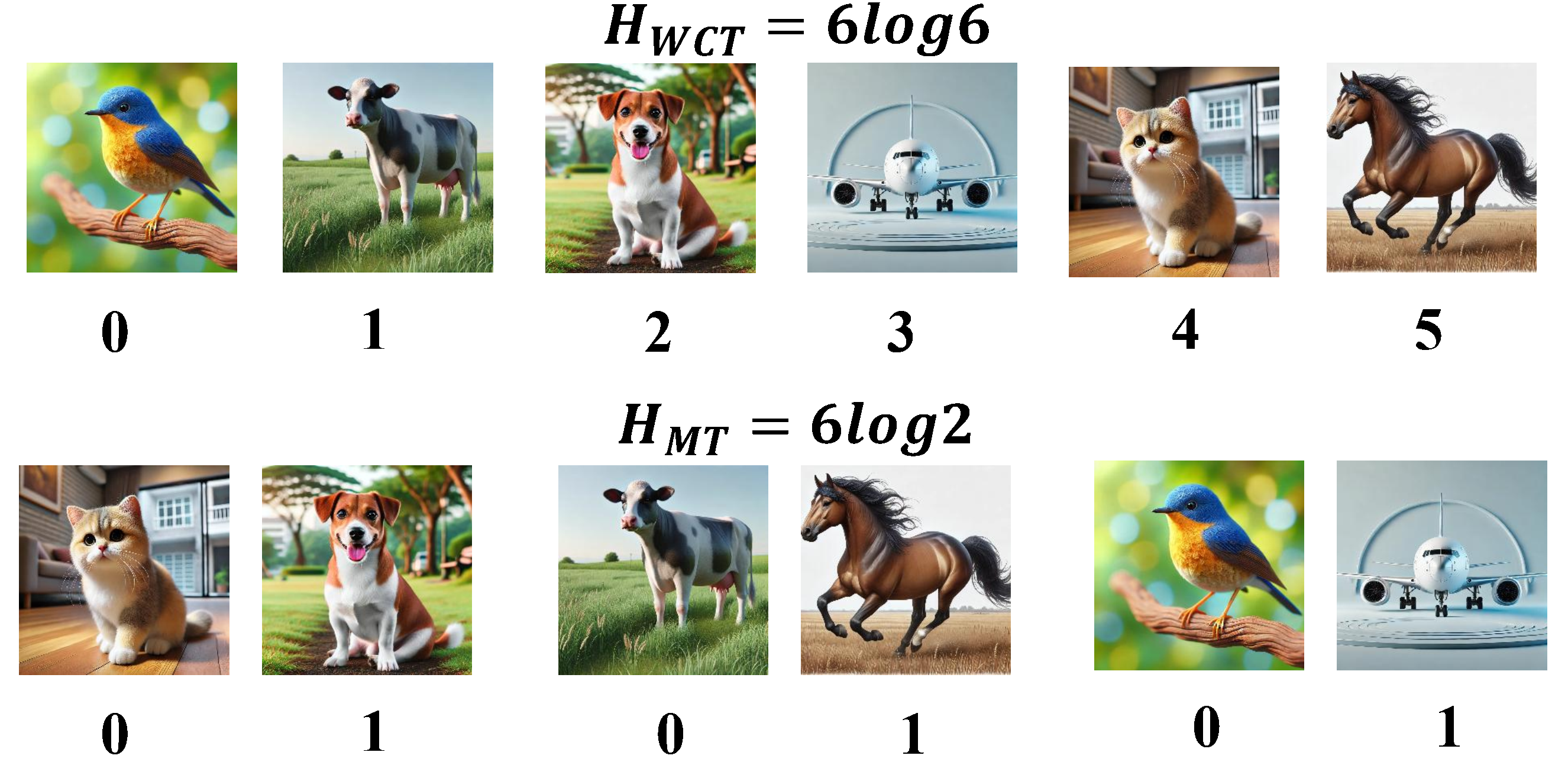}
\caption{Unfair comparison under the conventional supervised setting. The annotation cost varies across different training methods. $H$ represents the information entropy.}
\label{fig:two_secnario}
\vspace{-8pt}
\end{figure}

In this paper, (1) we establish a comparison framework based on the entropy-limited supervised setting and the uniform stability theory~\cite{stability_and_generalization,generalization_stable_2}. We prove that, for few-shot classification tasks, meta-learning has a tighter generalization error bound compared to WCT. Our experiments further support this conclusion. (2) We unravel two insights: meta-learning demonstrates more efficient utilization of entropy and exhibits greater robustness to label noise and heterogeneous tasks. These insights indicate that meta-learning is more suitable for unsupervised few-shot tasks. (3) We propose \ourmethod, \ie, \textbf{M}eta-learning \textbf{I}s \textbf{N}ot \textbf{O}ut, which enhances meta-learning in unsupervised few-shot and zero-shot tasks. The overview of \ourmethod~is shown in Figure~\ref{fig:overview}. The key ideas are to leverage the adaptive clustering algorithm DBSCAN~\cite{DBSCAN} for heterogeneous task learning and to utilize the dynamic head with an adaptive meta-scaler for robustness against label noise.

Our contribution can be summarized as follows:
\begin{itemize}
    \item We propose an entropy-limited supervised setting for fair comparison and theoretically prove that meta-learning has a tighter upper bound on the generalization error compared to WCT.
    \item We unravel two insights, \ie, meta-learning efficiently utilizes entropy and demonstrates greater robustness to label noise and heterogeneous tasks.
    \item We propose \ourmethod~for unsupervised zero-shot and few-shot classification tasks. It integrates DBSCAN for heterogeneous task learning and employs the dynamic head with an adaptive meta-scaler to handle label noise.
\end{itemize}

\section{Entropy-limited supervision}\label{sec:entropy_limited_supervision}
As shown in Figure~\ref{fig:two_secnario}, the conventional supervised setting necessitates that WCT accommodates more classes in comparison to meta-learning. Despite utilizing the same dataset, the WCT-based single-task learning method demands more annotation resources. Consequently, the classic supervised frameworks~\cite{a_good_embedding, meta_baseline, a_closer_again} do not establish an equitable basis for comparison. To address this problem, we introduce the entropy-limited supervised setting because the annotation process can be considered as an entropy-reduced process.
\begin{lemma}\label{lemma:correct_num}
Let the sample volume of the dataset be \(m\), the number of classes be \(C\), the sample number per class be balanced, and the entropy consumed by annotation be \(H\). Then, the expectation of correct labeled samples, i.e., $m'$, is given by
\begin{align}
 m' = \frac{m}{C}e^{\frac{H}{m}}~~~~~~~~\textit{s.t.}~~H \in [0 , mlogC].
\end{align}
\end{lemma}
The proof is detailed in Supplementary Material A. When $H \rightarrow mlogC$, $m' \rightarrow m$ and the entropy-limited setting degrades into the conventional supervised setting. When $H \rightarrow 0$, $m' \rightarrow \frac{m}{C}$. It means that samples are randomly labeled. This condition is equivalent to the unsupervised setting with unlabeled samples. Otherwise, the condition is equivalent to the supervised setting with label noise or the unsupervised setting with a pseudo-label generator.

\subsection{Comparison of Generalization Error}\label{sec:compare generalization error}
Based on the uniform stability theory, ~\citet{generalization_stable_2} derives a generalization error upper bound for WCT trained by single task learning and meta-learning. By introducing Lemma~\ref{lemma:correct_num}, we derive the generalization error bound within the entropy-limited setting.
\begin{theorem}[Generalization error of entropy-limited WCT]
Let the sample volume of the dataset be \( m \), the number of classes be $C_1$, the annotation entropy be $H$, and the single-task learning algorithm $\bm{A}$ have uniform stability $\beta$~\cite{generalization_stable_2}. Then the generalization error $R_{gen}(\bm{A})$ is bounded by the following equation with probability at least \( 1 - \delta \) for any \( \delta \in (0,1) \),
\begin{equation}
R_{gen}(\bm{A}) \leq 2\beta + (4m\beta+M)\sqrt{\frac{C_1\ln(1/\delta)}{2me^{H/m}}}.
\end{equation}
\end{theorem}
\begin{theorem}[Generalization error of entropy-limited meta-learning]\label{theorem:meta_gen_limited_entropy}
Let the sample volume of the dataset be \( m \), the number of classes per task be $C_2$, the number of samples per class be $k$, the number of tasks be $n$, the annotation entropy be $H$, the base-learner $\bm{A}$ has uniform stability $\beta$, and the meta-learner $\bm{\mathcal{A}}$ have uniform stability $\tilde{\beta}$~\cite{generalization_stable_2}. Then generalization error $R_{gen}(\bm{\mathcal{A}})$ is bounded by the following equation with probability at least \( 1 - \delta \) for any \( \delta \in (0,1) \),
\begin{equation}\label{eq:meta_generalization_bound}
R_{gen}(\bm{\mathcal{A}}) \leq 2\beta + 2\tilde{\beta}+(4n\tilde{\beta}+M)\sqrt{\frac{kC^2_2\ln(1/\delta)}{2me^{H/m}}}.
\end{equation}
\end{theorem}
The above theorems show that the dominant terms of the generalization error for meta-learning and WCT have a similar structure and are both expressed by $C_1$, $C_2$, $K$, $m$, and $H$. This means that, through the connection established by the entropy-limited setting, meta-learning and WCT can be fairly compared within a unified theoretical framework.

\begin{corollary}\label{corollary:main}
Let the base-level stability $\beta \sim o(\sqrt{1/m})$, the meta-level stability $\tilde{\beta} \sim o(\sqrt{1/n})$, and the entropy resource $H$ be equal for each algorithm. Then, the meta-learning algorithm $\bm{\mathcal{A}}$ has a tighter generalization error upper bound than the single-task learning algorithm $\bm{A}$ when
\begin{equation}\label{eq:theorem_result}
    C_2^2 \cdot k < C_1.
\end{equation}
\end{corollary}

The proof of theorem and corollary are detailed in Supplementary Material B. Taking the 5-way 1-shot task on the Omniglot dataset as an example, the total number of classes is $C_1=1628$. While for meta-learning, $C_2=5$ and $k=2$. This implies that $C_2^2 \cdot k = 50$ and the condition given by the above corollary is readily met in the few-shot classification setting. The number of tasks $n$ is replaced by a much smaller\footnote{Since the samples can be shared by different tasks, $n$ is much bigger than $\frac{m}{kC_2}$} term $\frac{m}{kC_2}$. As a result, the upper bound of the meta-learning algorithm $\bm{\mathcal{A}}$ is much tighter than the one we derived in Theorem~\ref{theorem:meta_gen_limited_entropy}. Therefore, meta-learning's bound is significantly tighter than that of WCT.

\subsection{Impact of Label Noise}\label{sec:label noise}
Lemma~\ref{lemma:correct_num} derives the ratio of label noise under the entropy-limited supervised setting. It enables a fair comparison of different models under equivalent entropy conditions.
\vspace{5pt}
\par
\noindent\textbf{Efficient Usage of Limited Entropy.}
We compare the classification performance of meta-learning and WCT under label noise interference by varying the entropy available for annotation. On Omniglot and Mini-Imagenet datasets, we use the same neural network architecture and learning configuration as~\cite{MAML, ANIL}. Figure~\ref{fig:entropy_vary} shows that meta-learning algorithms MAML~\cite{MAML} obtain higher accuracy under the same entropy. These results demonstrate that the meta-learning algorithm is more efficient under entropy-limited setting.
\vspace{5pt}
\par
\noindent\textbf{Bi-Level Optimization for Noise Robustness.}
In Figure~\ref{fig:entropy_vary}, the trend of the curves shows that when $H \in (6/8mlogC,logC)$, meta-learning is less susceptible to the interference of label noise. Table~\ref{tab:noise omniglot} and Table~\ref{tab:noise mini-imagenet} show the performance of WCT, ANIL~\cite{ANIL}, and MAML~\cite{MAML} under the same noise ratio. These results demonstrate that meta-learning algorithms exhibit stronger robustness to label noise.
To clarify meta-learning’s robustness to label noise, we observe its behavior across different neural network layers. Using SVCCA~\cite{SVCCA}, we measure representation stability across epochs, tracking how the learned representations evolve during training. The learned representation of $i-th$ layer can be written as $f_{\theta^i}(D)$, where $D$ is a fixed batch of samples. At epoch $t$, the representation stability of the $i-th$ layer can be defined as:
\vspace{-5pt}
\begin{align}
    rs^i_t=SVCCA(f_{\theta^i_t}(D),f_{\theta^i_{t-1}}(D)).
\end{align}
Figure~\ref{fig:representation stability} shows the results on the Omniglot dataset with 15\% noise. Due to the limitations of single-level optimization, all layers of WCT are disrupted by label noise, showing a significant stability reduction compared to the scenario without noise. In contrast, meta-learning algorithms confine the impact of noise to the task-specific "head" (L4) while maintaining cross-task representation stability in the "body" (L0-L3), aligning with the bi-level optimization principle.

\begin{figure}[t]
    \centering  
    \subfigure[Omniglot 5-way 1-shot]{\includegraphics[width=0.48\linewidth]{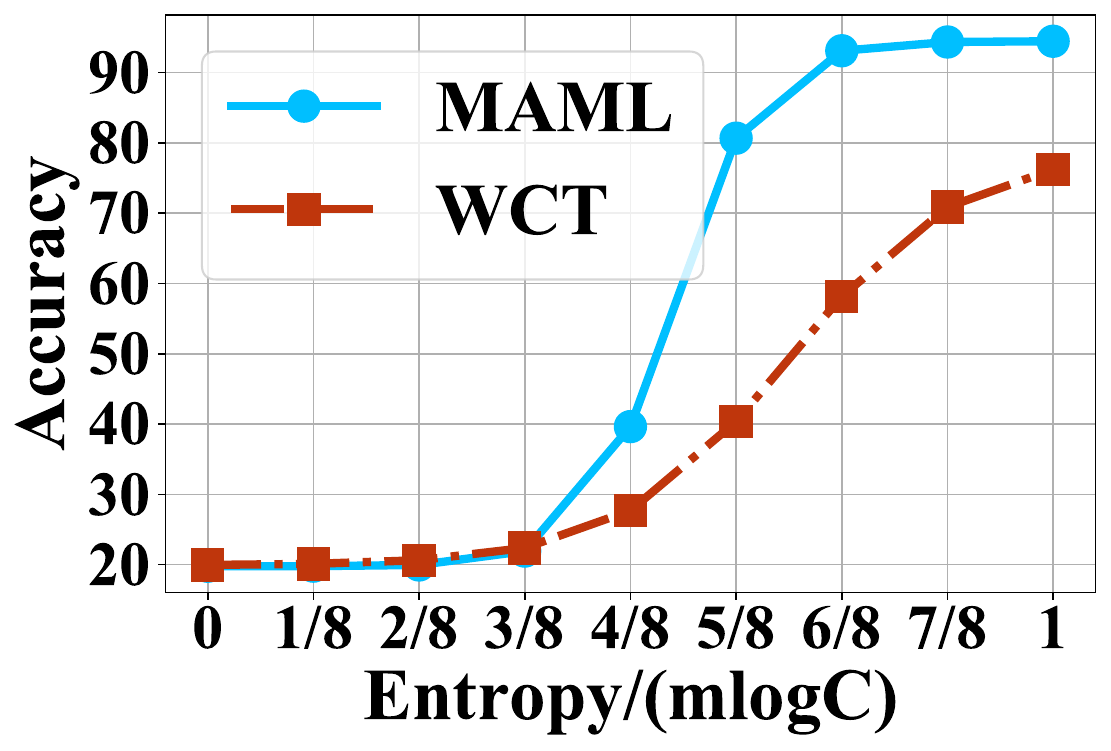}}
    \hspace{-3pt}
    \subfigure[Mini-Imagenet 5-way 1-shot]{\includegraphics[width=0.48\linewidth]{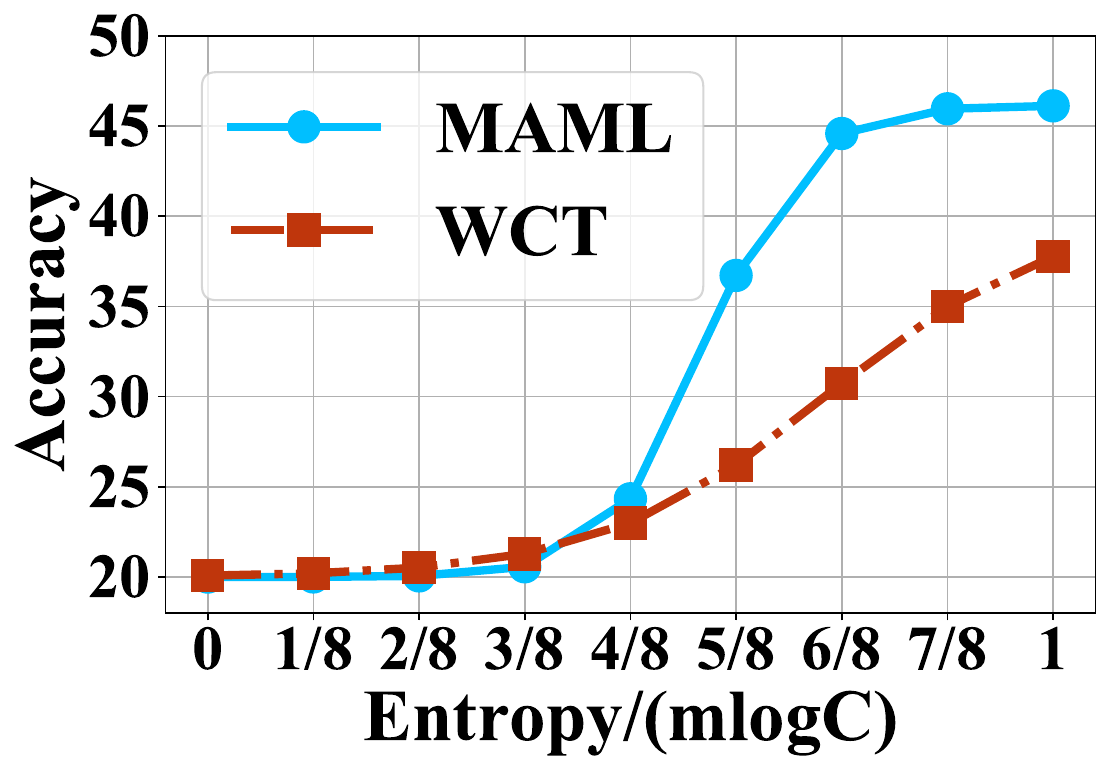}}
\caption{Comparing meta-learning and WCT by varying the entropy used for annotation.}
\label{fig:entropy_vary}
\end{figure}

\begin{table}[t]
  \centering
    \caption{Comparison of meta-learning's and WCT's resistance to label noise on the Omniglot dataset with 5-way 1-shot task. The metric is Accuracy\%.}
    \resizebox{\linewidth}{!}{
      \begin{tabular}{llccc}
        \toprule
         \textbf{Optimization Strategy}&Method& \textbf{0\% noise} & \textbf{15\% noise} & \textbf{30\% noise} \\
        \midrule
        Single-Level Optimization&WCT & 94.51 & 82.44 & 64.65 \\
        Bi-Level Optimization&ANIL & 94.35 & 91.72 & 80.59 \\
        Bi-Level Optimization&MAML & 94.46 & 91.58 & 80.72 \\
        \bottomrule
      \end{tabular}}
    \label{tab:noise omniglot}
\end{table}

\begin{table}[t]
    \centering
    \caption{Comparison of meta-learning's and WCT's resistance to label noise on the Mini-Imagenet dataset with 5-way 1-shot task. The metric is Accuracy\%.}
    \resizebox{\linewidth}{!}{
      \begin{tabular}{llccc}
        \toprule
         \textbf{Optimization Strategy} &Method& \textbf{0\% noise} & \textbf{15\% noise} & \textbf{30\% noise} \\
        \midrule
        Single-Level Optimization&WCT & 47.04 & 38.92 & 29.68 \\
        Bi-Level Optimization&ANIL & 46.77 & 41.69 & 37.45 \\
        Bi-Level Optimization&MAML & 46.81 & 41.63 & 37.51 \\
        \bottomrule
      \end{tabular}}
    \label{tab:noise mini-imagenet}
\end{table}

\begin{figure*}[t]
    \centering 
    \resizebox{0.7\linewidth}{!}{
    \includegraphics[width=\linewidth]{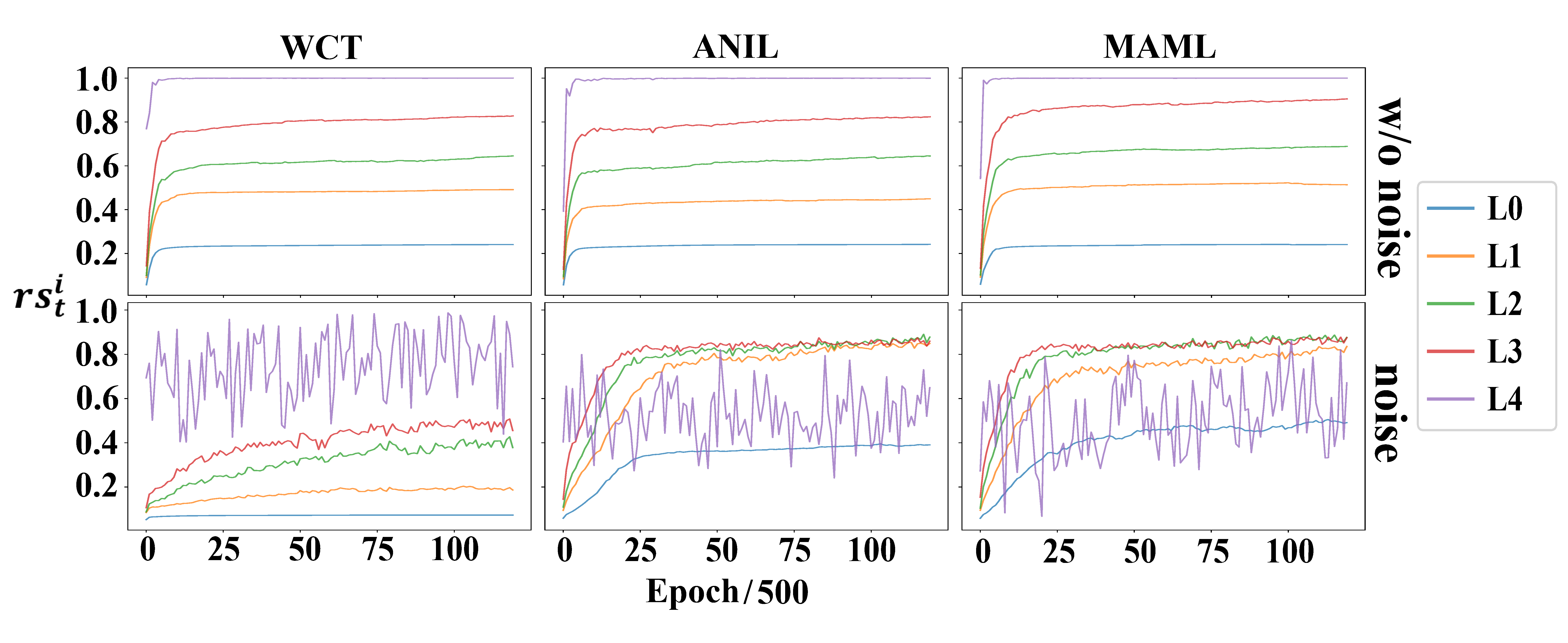}}
    \vspace{-10pt}
    \caption{Representation stability of WCT and meta-learning algorithms on the Omniglot dataset with 5-way 1-shot noisy tasks. The x-axis denotes epochs, and the y-axis denotes values of representation stability $rs^i_t$. A higher y-axis value means higher stability. L4 is the neural network's "head", and L0-L3 is the “body”.}
\label{fig:representation stability}
\end{figure*}

\begin{figure*}[t]
    \centering  
    \resizebox{0.65\linewidth}{!}{
    \includegraphics[width=\linewidth]{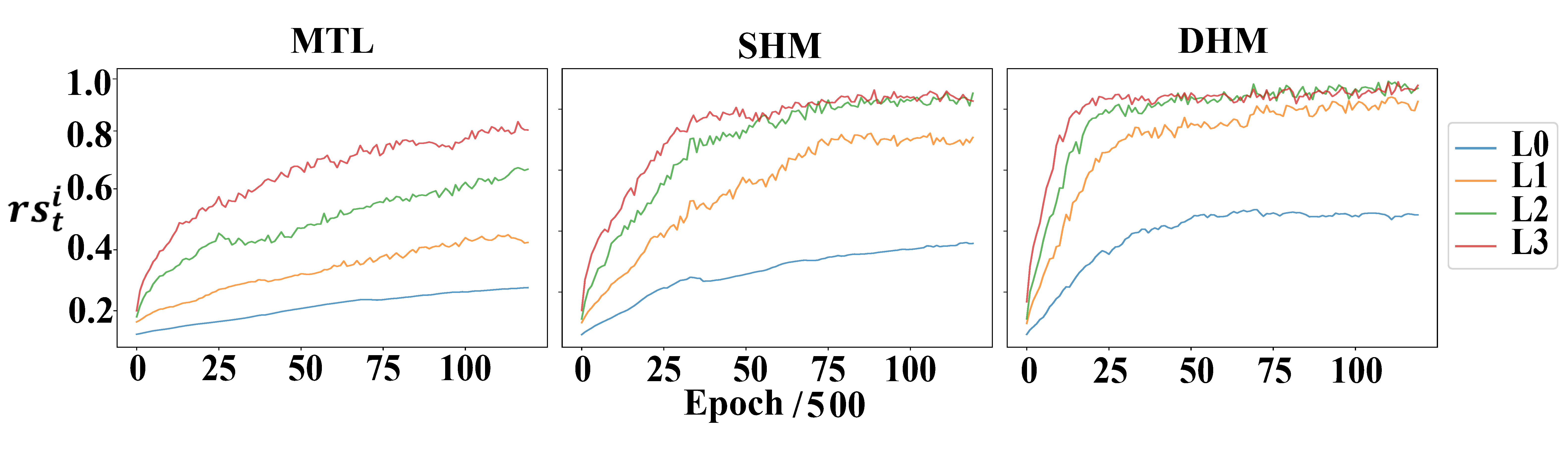}}
\vspace{-10pt}
\caption{Representation stability of DHM, SHM, and MTL trained on Omniglot's 5-20 classification way heterogeneous tasks.}
\label{fig:representation stability task type}
\end{figure*}

\subsection{Impact of Heterogeneous Tasks}\label{sec:heterogeneous tasks}
In the entropy-limited setting, samples from the same category cannot be precisely assigned to tasks. This means that the classification way is not fixed, inevitably making the tasks heterogeneous. Previous studies show that training a model with homogeneous tasks may lead to meta-overfitting\cite{Deconfounding,meta_without_memory,noise_data}, where the model performs well only on a specific classification task. Our subsequent analysis and experiments will demonstrate that a certain degree of task heterogeneity can improve the performance of meta-learning.

An intuitive way to handle heterogeneous tasks is to integrate the bi-level optimization group classification trick~\cite{any_way}, enabling dynamic-head meta-learning (DHM). We detail this approach in Section~\ref{sec:DHM}.
\vspace{5pt}
\par
\noindent\textbf{Heterogeneous Tasks Improve Performance.}
\begin{table}[t]
\centering
    \centering
    \captionsetup{justification=centering}
    \caption{Comparison of Accuracy\% on Omniglot's 5-20 way and Mini-Imagenet's 5-10 way heterogeneous Tasks.}
    \resizebox{\linewidth}{!}{
    \begin{tabular}{llcc}
    \toprule
     \textbf{Optimization Strategy}& \textbf{Method}&\textbf{Omniglot} & \textbf{Mini-Imagenet}\\
    \midrule
    Single-level Optimization&MTL & 72.95 & 35.86  \\
    Bi-level Optimization&SHM & 92.86  & 41.63  \\
    \rowcolor{gray!20}Bi-level Optimization&DHM & 93.27 & 44.09 \\
    \bottomrule
    \end{tabular}%
    }
  \label{tab:Dynamic_multi_type}%
\end{table}
We perform experiments on the Omniglot dataset with 5-20 classification ways, and the Mini-Imagenet dataset with 5-10 classification way. The setup is detailed in Supplementary Material C. Table~\ref{tab:Dynamic_multi_type} shows the comparison between dynamic-head meta-learning (DHM), classical static-head meta-learning(SHM), and classical multi-task learning (MTL). It can be found that the Bi-level-optimization-based approaches DHM and SHM perform better than the single-level optimization-based approach MTL. More importantly, DHM demonstrates improvements over SHM, indicating that it benefits from heterogeneous task learning.
\vspace{5pt}
\par
\noindent\textbf{Further Evidence.}  
To further support the above experimental results, we visualize the representation stability of the three models under Omniglot's 5-20 way heterogeneous task setting. As shown in Figure~\ref{fig:representation stability task type}, during the training process, DHM consistently achieves higher representation stability and better convergence efficiency.

\section{Method} \label{sec:method}
\begin{figure*}[t]
   \centering  
    \includegraphics[width=\linewidth]{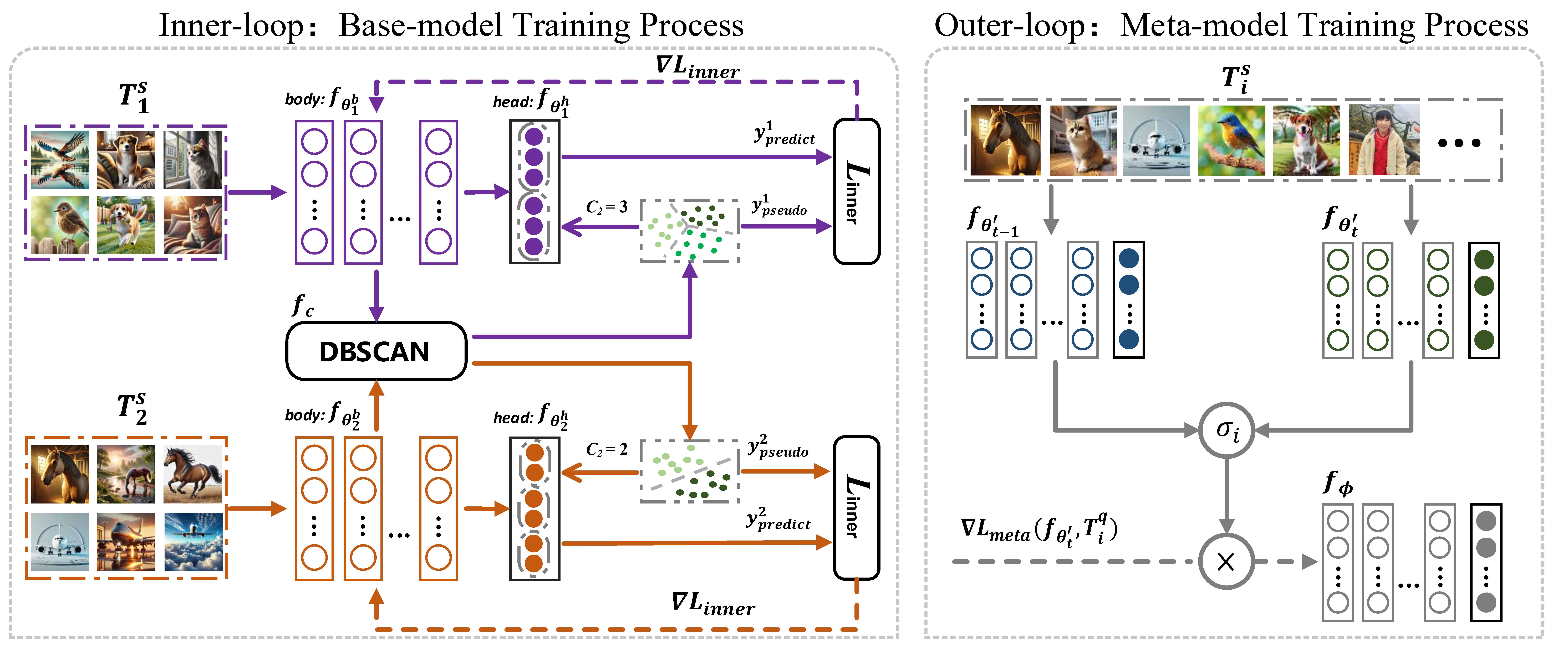}
\caption{Overview of MINO. (1) The left part shows the process of inner-loop. The base learner computes $y_{pseudo}$ and $y_{predict}$ by DBSCAN $f_c$ and a dynamic head $f_{\theta_i^h}$, respectively. Then, their cross-entropy is backpropagated. We apply a grouping classification technique at the head, \ie, $f_{\theta^h}$, to handle tasks with different numbers of clusters. (2) The right part shows the process of outer-loop. The meta gradient, \ie, $\nabla L_{meta}(f_{\theta'_t},T_i^q)$, is regulated by the meta-scaler, \ie, $\sigma_i$, where $\sigma_i$ is an adaptive scaler that operates based on the representation stability of $f_{\theta'_t}$ to ensure noise robustness.}
\label{fig:overview}
\end{figure*}

The analysis in Section~\ref{sec:entropy_limited_supervision} reveals several insights under the entropy-limited supervised setting.
\begin{enumerate}
    \item Meta-learning algorithms enable more efficient entropy utilization and better label noise robustness.
    \item Dynamic-head meta-learning can benefit from heterogeneous tasks.
    \item The representation stability $rs^i_t$ of the meta-model's head is sensitive to label noise, while the body is less affected.
\end{enumerate}

Figure~\ref{fig:overview} shows the overview of our method. Corresponding to each insight, (1) our method utilizes meta-learning for unsupervised few-shot tasks. (2) It leverages DBSCAN for heterogeneous task construction and the grouping classification trick~\cite{any_way} for heterogeneous task learning. (3) It employs a stability-based meta-scaler to regulate meta-gradient, enhancing the robustness to label noise.

\subsection{Meta-Learning}\label{sec:train_and_test}
During the training phase, by sampling dataset $D$, we obtain $\{T_1, T_2, ..., T_n\}$ as the meta-sample, where $T_i$ can be divided to support-set $T^s_i$ and query-set $T^q_i$~\cite{MAML}. For each $T_i$, we clone a base-model $f_{\theta_i}$ from the meta-model $f_{\phi}$. When all inner-loop rounds are completed and all tasks are learned, we perform an update to optimize $f_{\phi}$ and complete one step of outer-loop. The optimization problem corresponding to training a meta-model $f_{\phi}$ can be written as:
\vspace{-5pt}
\begin{align}
    \argmin_{\phi} \sum_{T_i \in D}L_{meta}(f_{\theta'_i}, T^q_i),
\end{align}
where $f_{\theta'_i}$ is the updated base model derived by Equation~\ref{eq:inner_update}. $L_{meta}$ is the query loss derived by Equation~\ref{eq:qry_loss}. The gradient-based optimization process can be written as:
\vspace{-5pt}
\begin{align}
\label{eq:outer_update}
f_{\phi}=f_{\phi}-\frac{\eta}{n}\nabla_{\phi}{\sum_{i=1}^{n}\sigma_i L_{meta}(f_{\theta'_i}, T^q_i)},
\end{align}
where $\eta$ is the meta-learning rate $\sigma_i$ is an adaptive scaler based on representation stability.

When all outer-loops are completed, we obtain a well-trained neural network $f_{\phi^*}$. During the testing phase, our method can complete the evaluation in an unsupervised manner. Depending on the requirements of the few-shot and zero-shot tasks, we can choose to perform or not to perform inner-loop to fine-tune $f_{\phi^*}$. 

\subsection{Unsupervised Heterogeneous Task Construction}\label{sec:UHT}
The process of constructing heterogeneous tasks is shown in the left part of Figure~\ref{fig:overview}. For the samples in $T_i$, we use the pretrained $f_{\theta^b}$ to project them in an embedding space and use a clustering algorithm $f_c$, to divide the embeddings into multiple categories. We use DBSCAN to dynamically partition clusters for adaptively constructing heterogeneous tasks. Compared to the commonly-used fixed clustering method K-means, DBSCAN alleviates meta-overfitting~\cite{Deconfounding,meta_without_memory} caused by homogeneous tasks. The pseudo label can be calculated by 
$\tilde{y}=f_b \circ f_{\theta^c}(x)$, and the predicted label can be calculated by $\bar{y}=f_{\theta^b} \circ f_{\theta^h}(x)$, where $\circ$ represents the composite function notation. The support loss of inner-loop can be written as:
\small
\begin{align}\label{eq:inner_loss}
L_{inner}(f_{\theta_i}, T^s_i)=\sum_{x \in T^s_i}L(f_{\theta^h_i} \circ f_{\theta^b_i}(x), f_{\theta^h_i} \circ f_c(x)).
\end{align}
\normalsize
A base learner after one inner-loop update can be written as
\begin{align}
\label{eq:inner_update}
{f_{\theta'_i}}={f_{\theta_i}}-\alpha\nabla L_{inner}(f_{\theta_i}, T^s_i).
\end{align}
The query loss can be written as
\small
\begin{equation}\label{eq:qry_loss}
L_{meta}(f_{\theta'_i}, T^q_i)=\sum_{x \in T^q_i}L(f_{{\theta'}^{h}_i} \circ f_{{\theta'}^{b}_i}(x), f_{{\theta'}^{h}_i} \circ f_c(x)).
\end{equation}
\normalsize
To avoid overfitting to sampling bias, we drop the clusters with relatively small scale, which is controlled by the hyperparameter ``min\_sample" in DBSCAN.

\begin{algorithm}[t]
\small
\caption{Dynamic Meta-Learning}
\label{alg:dynamic_meta}
\centering
\begin{tabular}{l}
\textbf{Require:} Meta-Dataset $D$; Meta-model $f_\phi=\{f_{\phi^h},f_{\phi^b}\}$; \\
\hspace{12mm} Clustering $f_c$; Meta-Learning rate $\eta$; \\
\hspace{12mm} Base-Learning rate $\alpha$. \\
\textbf{While} not done \textbf{do} \\
\hspace{5mm} Sample $\{T_1,...,T_n\}$ from $D$ \\
\hspace{5mm} \textbf{For each} $T_i$ \textbf{do} \\
\hspace{10mm} $f_{\theta_i} \leftarrow f_{\phi}$ \\
\hspace{10mm} \textbf{For each} Inner-loop \textbf{do} \\
\hspace{15mm} $L_{inner}(f_{\theta_i}, T^s_i) \leftarrow$ Equation~\ref{eq:inner_loss} \\
\hspace{15mm} $f_{\theta_i} \leftarrow f_{\theta_i}-\alpha\nabla_{\theta_i} L_{inner}(f_{\theta_i}, T^s_i)$ \\
\hspace{10mm} \textbf{End For} \\
\hspace{10mm} $L_{{meta}_i} \leftarrow$ Equation~\ref{eq:qry_loss} \\
\hspace{5mm} \textbf{End For} \\
\hspace{5mm} $f_{\phi}=f_{\phi}-\frac{\eta}{n}\nabla_{\phi}{\sum_{i=1}^{n}\sigma_i L_{Meta_i}}$ \\
\textbf{End While} \\
\end{tabular}
\end{algorithm}

\begin{table*}[t]
  \centering
  \caption{Accuracy in \% on unsupervised few-shot setting}
    \resizebox{\linewidth}{!}
    {
            \begin{tabular}{lcccccccccc}
    \toprule
          & \multicolumn{4}{c}{\textbf{Omniglot}} & \multicolumn{4}{c}{\textbf{Mini-Imagenet}} & \multicolumn{2}{c}{\textbf{Tiered-Imagenet}}\\
    \midrule
    \textbf{(way, shot)} & \textbf{(5, 1)} & \textbf{(5, 5)} & \textbf{(20, 1)} & \textbf{(20, 5)} & \textbf{(5, 1)} & \textbf{(5, 5)} & \textbf{(5, 20)} & \textbf{(5, 50)} & \textbf{(5, 1)} & \textbf{(5, 5)} \\ 
    UMTRA~\cite{UMTRA} & 82.97 ± 0.68 & 94.84 ± 0.60 & 73.51 ± 0.53 & 91.22 ± 0.59 & 39.14 ± 1.02 & 49.21 ± 0.90 & 57.66 ± 1.02 & 59.68 ± 1.17 & 41.03 ± 1.00 & 51.07 ± 0.92 \\
    CACTUs-MA-DC~\cite{CACTUs} & 67.98 ± 0.80 & 87.07 ± 0.63 & 47.48 ± 0.59 & 72.21 ± 0.54 & 39.11 ± 1.08 & 53.40 ± 0.88 & 63.00 ± 1.06 & 68.62 ± 1.02 & 41.00 ± 1.13 & 55.26 ± 0.82 \\
    CACTUs-Pr-DC~\cite{CACTUs} & 67.08 ± 0.72 & 82.97 ± 0.64 & 46.32 ± 0.51 & 65.75 ± 0.62 & 38.47 ± 1.14 & 53.01 ± 0.91 & 61.05 ± 1.09 & 62.82 ± 1.08 & 40.36 ± 1.17 & 54.87 ± 0.97 \\
    CACTUs-MA-Bi~\cite{CACTUs} & 57.84 ± 0.75 & 78.12 ± 0.67 & 34.98 ± 0.57 & 57.75 ± 0.58 & 36.13 ± 1.07 & 50.45 ± 0.90 & 60.97 ± 1.16 & 66.34 ± 1.14 & 38.02 ± 1.08 & 52.31 ± 0.94 \\
    CACTUs-Pr-Bi~\cite{CACTUs} & 53.58 ± 0.65 & 71.21 ± 0.68 & 32.79 ± 0.53 & 50.12 ± 0.51 & 36.05 ± 1.06 & 49.87 ± 0.92 & 58.47 ± 0.98 & 62.56 ± 1.20 & 37.94 ± 1.02 & 51.73 ± 0.89 \\
    PsCo~\cite{PsCo} & 93.25 ± 0.59 & 97.56 ± 0.34 & 82.06 ± 0.43 & 91.01 ± 0.45 & 42.90 ± 0.95 & 54.87 ± 0.94 & 65.66 ± 1.05 & 69.94 ± 1.11 & 44.79 ± 0.91 & 56.73 ± 0.90 \\
    Meta-GMVAE~\cite{Meta-GMVAE} & 93.81 ± 0.75 & 96.85 ± 0.50 & 81.29 ± 0.62 & 89.00 ± 0.51 & 41.78 ± 1.13 & 54.15 ± 0.87 & 62.11 ± 1.14 & 67.11 ± 1.21 & 43.67 ± 1.08 & 56.01 ± 0.83 \\
    \rowcolor{gray!20}MINO & 93.75 ± 0.46 & 97.71 ± 0.37 & 83.57 ± 0.41 & 94.69 ± 0.40 & 44.73 ± 1.01 & 60.38 ± 0.89 & 69.94 ± 0.95 & 73.39 ± 1.07 & 46.95 ± 1.05 & 62.14 ± 0.76 \\
    \midrule
    MAML (supervised)~\cite{MAML} & 94.46 & 98.83 & 84.6  & 96.29 & 46.81 & 62.13 & 71.03 & 75.54 & 48.70 & 63.99 \\
    \bottomrule
    \end{tabular}%
    }
  \label{tab:comp_unsup_few_class}%
\end{table*}%

\begin{table*}[t]
  \centering
  \caption{Accuracy in \% on unsupervised zero-shot setting}
  \resizebox{0.75\linewidth}{!}{
    \begin{tabular}{lcccccc}
    \toprule
     \textbf{Method}& \textbf{CIFAR-10} & \textbf{CIFAR-100} & \textbf{STL-10} & \textbf{ImageNet} & \textbf{Tiny-MINIST} & \textbf{DomainNet} \\
    \midrule
    DeepCluster~\cite{DeepCluster} & 63.02 ± 1.14 & 35.05 ± 1.11 & 52.21 ± 1.42 & 24.83 ± 0.95 & 78.63 ± 1.68 & 18.09 ± 0.88 \\
    IIC~\cite{IIC} & 64.05 ± 1.02 & 36.23 ± 1.27 & 53.78 ± 1.30 & 25.07 ± 0.88 & 79.21 ± 1.54 & 18.18 ± 0.74 \\
    MAE~\cite{MAE} & 68.83 ± 1.19 & 39.11 ± 1.52 & 56.19 ± 1.47 & 27.32 ± 1.14 & 81.03 ± 1.36 & 20.53 ± 1.03 \\
    NVAE~\cite{NVAE} & 67.43 ± 1.37 & 38.29 ± 1.45 & 55.78 ± 1.22 & 27.21 ± 0.98 & 81.52 ± 1.61 & 19.84 ± 0.79 \\
    BiGAN~\cite{BiGAN} & 67.61 ± 1.24 & 38.78 ± 1.19 & 55.24 ± 1.34 & 26.85 ± 1.07 & 80.09 ± 1.27 & 19.23 ± 0.95 \\
    ReSSL~\cite{ReSSL} & 70.27 ± 1.15 & 41.48 ± 1.60 & 58.52 ± 1.31 & 31.25 ± 1.13 & 83.17 ± 1.24 & 21.42 ± 0.92 \\
    Meta-GMVAE~\cite{Meta-GMVAE} & 71.73 ± 1.28 & 41.26 ± 1.02 & 58.69 ± 1.51 & 30.08 ± 1.57 & 84.65 ± 1.03 & 21.06 ± 1.37 \\
    MINO-kmeans & 69.06 ± 1.34 & 39.55 ± 1.27 & 57.05 ± 1.49 & 29.89 ± 1.83 & 83.15 ± 1.01 & 19.68 ± 1.16 \\
    \rowcolor{gray!20} MINO & 73.15 ± 1.09 & 43.34 ± 1.41 & 60.74 ± 1.35 & 31.12 ± 1.06 & 86.45 ± 1.28 & 22.68 ± 0.91 \\
    \bottomrule
    \end{tabular}}
  \label{tab:comp unsup class}%
\end{table*}%

\subsection{Dynamic Head with Stability-Based Meta-Scaler}\label{sec:DHM}
As shown in the left part of Figure~\ref{fig:overview}, the dynamic head $f_{\theta^b}$ employs the grouping classification trick~\cite{any_way} to adapt to heterogeneous tasks. According to the clustering counts $C_2$ given by DBSCAN, the dynamic head divides its classifier layer into multiple groups. Each group acts as an individual classifier.

As mentioned in Figure~\ref{fig:representation stability task type}, the task-specific head in meta-learning exhibits high representation stability without noise and low stability with noise. Therefore, we introduce a stability-based meta-scaler $\sigma_i$ to control the meta-gradient adaptively.
\begin{equation}
    \sigma_i = SVCCA(f_{\theta'_t}(T_i),f_{\theta'_{t-1}}(T_i)).
\end{equation}
The right part of Figure~\ref{fig:overview} shows that the meta-scaler $\sigma_i$ can limit the impact of the meta-gradient $\nabla L_{meta}(f_{\theta'_t} , T_i^q)$ on the meta-model $f_{\phi}$ according to noise level, thereby improving robustness to label noise. The process of MINO is outlined in Algorithm~\ref{alg:dynamic_meta}.

\section{Experiment}\label{sec:experiment}
The experimental evaluation is conducted on a high-performance computing platform equipped with dual A100 GPUs, an Intel Xeon Gold 6348 processor, and 512 TB of DDR4 memory. To ensure statistical validity, we implement over 5 independent tests for each experiment and calculate the mean and standard deviation of the results. The subsequent section presents our primary experimental findings and setups, with comprehensive technical details available in Supplementary Material C.

\begin{table}[t]
    \centering
    \caption{Unsupervised Datasets Description}
    \resizebox{\linewidth}{!}{
    \begin{tabular}{lcc}
        \toprule
        \textbf{Dataset} & \textbf{Classes Number $C_1$} & \textbf{Sample Number $m$} \\
        \midrule
        CIFAR-10 & 10 & 60000 \\
        CIFAR-100 & 100 & 60000 \\
        STL-10 & 10 & 130000 \\
        ImageNet & 1000 & 14197122 \\
        Tiny-ImageNet & 200 & 100000 \\
        DomainNet & 345 & 600000 \\
        Omniglot & 1623 & 32460 \\
        Mini-ImageNet & 100 & 60000 \\
        Tiered-ImageNet & 608 & 779165 \\
        \bottomrule
    \end{tabular}
    }
    \label{tab:Datasets}
\end{table}

\subsection{Unsupervised few-shot classification}\label{sec:unsupervised few-shot}
\noindent\textbf{Task.} 
Following the setup provided by~\citet{MAML}, we train and evaluate models on disjoint meta-training and meta-testing tasks. During the evaluation stage, we use the support set in meta-test tasks to fine-tune the models and then compare the accuracy of each model.
\vspace{5pt}
\par
\noindent\textbf{Datasets.}
We compare the models on Omniglot~\cite{omniglot}, Mini-Imagenet~\cite{miniImagenet}, and Tiered-Imagenet~\cite{tiered-Imagenet} datasets. The way and shot of the classification tasks are shown in Table~\ref{tab:comp_unsup_few_class}. The classes of the meta-training and the meta-testing tasks are disjoint from each other.
\vspace{5pt}
\par
\noindent\textbf{Baselines.} 
We compare MINO with several state-of-the-art few-shot unsupervised meta-learning classification algorithms, including CACTUs~\cite{CACTUs}, UMTRA~\cite{UMTRA}, Meta-GMVAE~\cite{Meta-GMVAE}, and PsCo~\cite{PsCo}. Using the meta-learning training methodology and unsupervised embedding algorithms, CACTUs manifests in four distinct implementations: CACTUs-MA-DC, CACTUs-Pr-DC, CACTUs-MA-Bi, and CACTUs-Pr-Bi. Pr denotes ProtoNet~\cite{ProtoNet}, MA signifies MAML, DC indicates DeepCluster, and Bi represents BiGAN. MINO, CACTUs, and UMTRA follow the same MAML setting given by~\citet{MAML}. PsCo and GMVAE reuse the same setting given by~\citet{PsCo} and~\citet{Meta-GMVAE}.
\vspace{5pt}
\par
\noindent\textbf{Configuration.} 
For MINO, we follow the baseline architecture given by~\citet{MAML}. We set epoch, inner-loop learning rate, outer-loop learning rate, meta-batch size, inner-loop step, and number of sample per task, as 30000, 0.05, 0.001, 8, 5, and 50 respectively. For DBSCAN, we set min\_samples and eps as 15 and 1.0, respectively. We maintain this setting in all subsequent experiments. For more details, please refer to Supplementary Material C.
\vspace{5pt}
\par
\noindent\textbf{Results.} 
Table~\ref{tab:comp_unsup_few_class} shows that, compared to the second-highest performing algorithm, PsCo, MINO achieves superior performance with a 2.85\% increase in accuracy, averaged across 10 tasks. Although all methods are based on meta-learning, CACTUs and UMTRA perform significantly worse than GMVAE and PsCo. This is because CACTUs and UMTRA heavily rely on the quality of generated pseudo-labels and do not enhance their robustness to label noise. In contrast, Meta-GMVAE and PsCo reduce dependence on annotation quality by leveraging autoencoder and contrastive learning strategies, respectively.

\subsection{Unsupervised Zero-shot classification}\label{sec:unsupervised zero-shot}
\noindent\textbf{Task.} 
In this task, we train and evaluate models on disjoint pre-training and evaluation datasets. In the evaluation phase, we do not use any labels to adjust models. We assess the models' zero-shot classification performance concerning accuracy metrics.
\vspace{5pt}
\par
\noindent\textbf{Datasets.}
We compare the models on CIFAR-10, CIFAR-100, STL-10, Imagenet, Tiny-Imagenet, and DomainNet. DomainNet is a dataset used for domain generalization tasks. We use one domain for testing and the remain 5 domains for both training and validation. We follow the protocol given by~\cite{ReSSL} for CIFAR-10, CIFAR-100, STL-10, Imagenet, and Tiny-Imagenet.
\vspace{5pt}
\par
\noindent\textbf{Baselines.}
We compare MINO with several unsupervised zero-shot learning algorithms, including DeepCluster~\cite{DeepCluster}, IIC~\cite{IIC}, MAE~\cite{MAE}, NVAE~\cite{NVAE}, BiGAN~\cite{BiGAN}, ReSSL~\cite{ReSSL}, and MINO with K-means. 
\vspace{5pt}
\par
\noindent\textbf{Results.}
Table~\ref{tab:comp_unsup_few_class} shows that, compared to the second-highest performing method Meta-GMVAE, MINO demonstrates superior accuracy with a 1.70\% improvement averaged across the 6 datasets. A notable comparison can be drawn between MINO-kmeans and DeepCluster, as both methodologies utilize clustering-based pseudo-label generators. The functionality of these generators is fundamentally limited by the annotation capabilities of K-means clustering, which operates within the constraints of an entropy-limited supervised setting. Consequently, the WCT-trained DeepCluster demonstrates reduced effectiveness. The autoencoder architectures MAE, NVAE, and BiGAN exhibit similar performance characteristics, possibly because their autoencoder structure causes the model to overfit to low-level features. ReSSL, implementing a relation-based self-supervised methodology, emphasizes inter-instance relationships. Both ReSSL and Meta-GMVAE bypass full-category annotation constraints, resulting in enhanced performance metrics.

\subsection{Comparison of WCT and Meta-Training}
\noindent\textbf{Evidence for Corollary.}
To further support Corollary~\ref{corollary:main}, we vary $C_2$ and $k$ to observe the performance of MAML and WCT under both entropy-limited, \ie, $H=mlog5$, and conventional supervised settings. To simulate the most challenging scenario, we ensure that samples are not shared between tasks, \ie, $n=\frac{m}{kC_2}$. Hyperparameters $C_1$, $C_2$, and $k$ are utilized during the training process, with the corresponding results shown in Figure~\ref{fig:entropy_vary} obtained from Omniglot 5-way 1-shot testing tasks. Figure~\ref{fig:entropy_vary} shows that the accuracy of entropy-limited MAML decreases as $C_2$ and $k$ increase. Furthermore, it gradually converges to the baseline given by entropy-limited WCT. These results further validate Corollary~\ref{corollary:main}. Analysis reveals that when $k$ increases to 10, MAML demonstrates a notable enhancement in accuracy. This improvement can be attributed to the more optimal distribution of sample numbers, effectively reducing bias introduced by small query sets~\cite{how_distribute}. In addition, the increasing values of $C_2$ and $k$ lead to a progressive expansion of the performance differential between MAML and entropy-limited MAML, highlighting the impact of the entropy-limited supervised setting on algorithmic performance.
\begin{figure}[t]
    \centering  
    \includegraphics[width=\linewidth]{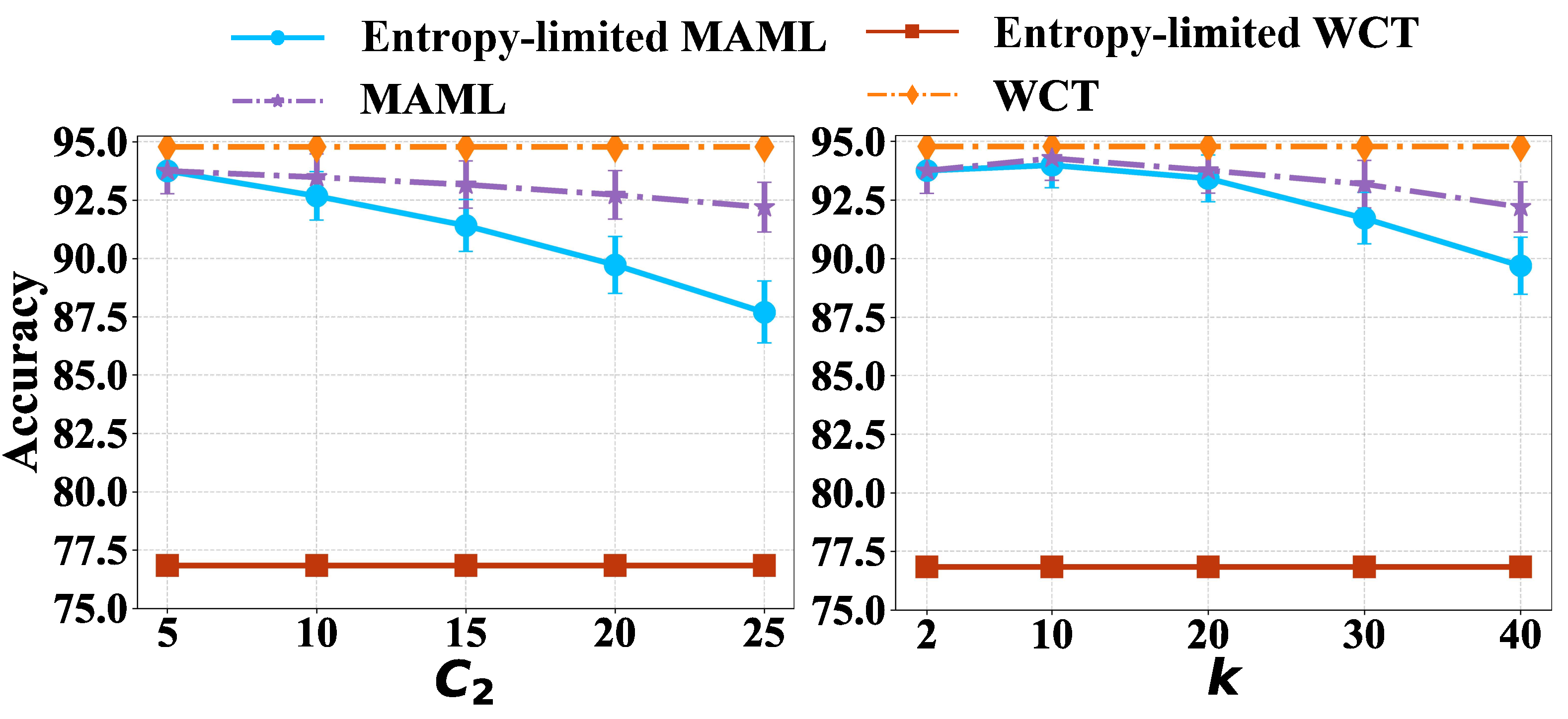}
\caption{Impact of $C_2$ and $k$ on entropy-limited and conventional supervised settings.}
\label{fig:entropy_vary}
\end{figure}

\vspace{5pt}
\par
\noindent\textbf{3D Few-Shot Classification.}
\begin{table}[t]
  \centering
  \caption{Accuracy in \% on the entropy-limited supervised setting for 3D few-shot classification tasks.}
    \resizebox{\linewidth}{!}
    {
            \begin{tabular}{lcccc}
    \toprule
          & \multicolumn{2}{c}{\textbf{ModelNet40}} & \multicolumn{2}{c}{\textbf{ShapeNetCore}} \\
    \cmidrule(lr){2-3} \cmidrule(lr){4-5}
    \textbf{(way , shot)} & \textbf{(5,1)} & \textbf{(5,5)} & \textbf{(5,1)} & \textbf{(5,5)} \\ 
    \midrule
    Entropy-limited WCT & 48.11 ± 0.75 & 56.28 ± 0.60 & 49.62 ± 0.68 & 62.95 ± 0.59 \\
    Entropy-limited MAML& 52.46 ± 0.72 & 59.46 ± 0.64 & 54.03 ± 0.71 & 65.32 ± 0.57 \\
    \bottomrule
    \end{tabular}%
    }
  \label{tab:3D_few_shot}%
\end{table}%
We compare WCT and MAML on ModelNet40 and ShapeNetCore~\cite{MobileNet}, which are 3D datasets. Setting $H=mlog5$ to simulate an entropy-limited supervised setting, we evaluate model performance on 5-way few-shot tasks. We use MVCNN~\cite{MVCN} to project 3D samples into 2D and adopt algorithmic setting as detailed in the above section to maintain configuration consistency. For few-shot tasks, we follow the setting given by~\citet{3D_baseline}. Table~\ref{tab:3D_few_shot} shows that the meta-learning algorithm MAML outperforms WCT under the entropy-limited setting. This result suggests that the Corollary~\ref{corollary:main} also holds for 3D few-shot classification tasks.

\subsection{Ablation Study}\label{sec:ablation}
We perform ablation experiments on unsupervised few-shot and zero-shot datasets, including Omniglot, CIFAR-100, and STL-10. Our experimental design comprises three distinct control groups. The first control group utilizes K-means clustering to generate homogeneous tasks, employing a static head for task learning. For the second group, we implement WCT in place of meta-learning methodology. The third group operates without the meta-scaler component for adaptive meta-updates. We maintain dataset parameters and algorithmic configurations as detailed in the above section. The results shown in Table~\ref{tab:ablation} indicate that MINO achieves superior performance, thereby validating the importance of heterogeneous task construction, meta-learning paradigm, and stability-based meta-scaler.

\subsection{Sensitivity Study}\label{sec:sensitive}
\noindent\textbf{Hyperparameters Selection.}
We perform sensitivity experiments on Omniglot and CIFAR-100. While meta-learning methods and unsupervised methods traditionally face challenges with hyperparameter tuning, our proposed method introduces only two hyperparameters. Our analysis examines the influence of the \textit{eps} (scanning radius of DBSCAN) and \textit{min\_samples} (the minimum number of samples within a cluster) parameters on overall performance. As shown in Figure~\ref{fig:sensitive}, \ourmethod~ consistently achieves robust and optimal performance when eps falls within $[10,20]$ and \textit{min\_samples} within $[0.5,1.5]$. These findings indicate that \ourmethod~exhibits minimal sensitivity to hyperparameter adjustments, lending itself well to practical implementation.

\begin{table}[t]
    \centering
        \caption{Effectiveness of each component. We compared the classification Accuracy\% on the unsupervised few-shot and zero-shot scenarios.}
        \resizebox{\linewidth}{!}{
            \begin{tabular}{lcccccc}
            \toprule
                  & \multicolumn{4}{c}{\textbf{Omniglot}} & \textbf{CIFAR-100} & \textbf{STL-10} \\
            \midrule
             \textbf{(w, s)}& \textbf{(5, 1)} & \textbf{(5, 5)} & \textbf{(20, 1)} & \textbf{(20, 5)} & \textbf{(100,0)} & \textbf{(10,0)} \\
            \textbf{W/O DBSCAN} & 87.12 & 92.67 & 73.49 & 80.73 & 37.58 & 52.27 \\
            \textbf{W/O Meta-Learning} & 74.32 & 90.91 & 51.83 & 77.42 & 32.37 & 47.75 \\
            \textbf{W/O Meta-Scaler} & 91.56 & 94.12 & 79.68 & 87.21 & 40.19 & 56.84 \\
            \textbf{Ours} & 93.81 & 96.85 & 81.29 & 89.00 & 42.34 & 58.74 \\
            \bottomrule
            \end{tabular}
        }
        \label{tab:ablation}
\end{table}

\begin{figure}[t]
    \centering
        \centering
        \includegraphics[width=\linewidth]{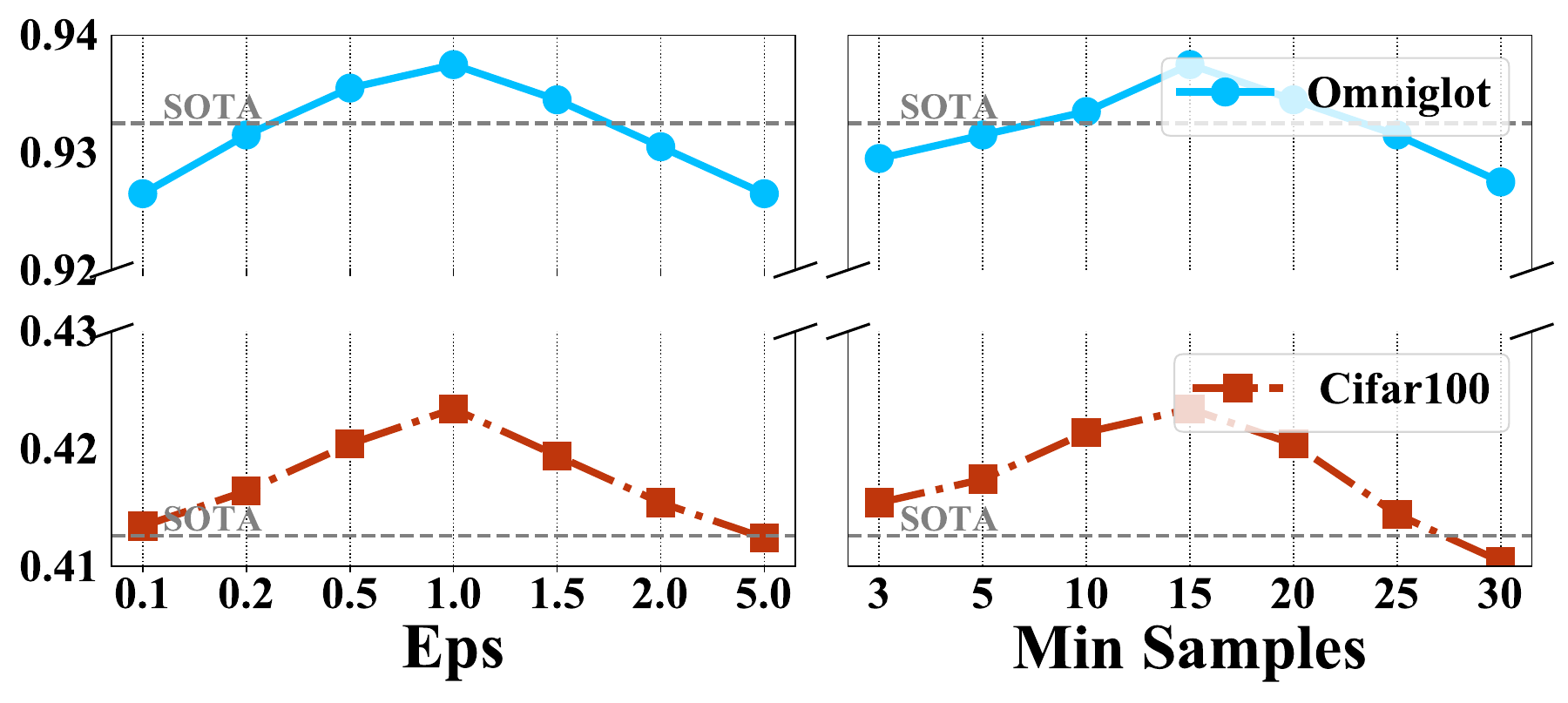}
        \caption{Sensitivity of MINO to new hyperparameters on unsupervised learning and unsupervised few-shot learning datasets.}
        \label{fig:sensitive}
\end{figure}
\vspace{5pt}
\par
\noindent\textbf{The Update Strategy Study.}
Our analysis examines the performance implications of implementing MAML-style versus ANIL-style update strategies for MINO, specifically regarding computational efficiency and accuracy metrics. Our empirical evaluation, shown in Table~\ref{tab:overhead tradeoff}, demonstrates that for the Omniglot dataset with 5-way 1-shot classification, the ANIL strategy implementation yields reduced computational time within the same target accuracy. Nevertheless, MINO-MAML demonstrates superior accuracy when the learning process converges. Therefore, in practice, we need to choose the update strategy based on the accuracy and overhead trade-off. Note that the above results differ from those in~\cite{ANIL}, possibly due to the complexity of the unsupervised tasks. Updating only the head in the inner loop is inadequate for learning unsupervised heterogeneous tasks. In scenarios where high accuracy is required, a MAML-style update strategy appears to be more effective.
\begin{table}[t]
  \centering
  \caption{Comparison of MINO-MAML and MINO-ANIL on accuracy vs. training-time trade-off.}
          \resizebox{0.65\linewidth}{!}{
    \begin{tabular}{ccc}
    \toprule
    Accuracy\% & MINO-ANIL & MINO-MAML \\
    \midrule
    50.0    & 655    & 748 \\
    60.0    & 1429   & 1681 \\
    70.0    & 2730   & 2992 \\
    80.0    & 5564   & 5642 \\
    90.0    & 13987  & 14723 \\
    93.5  & /     & 25717 \\
    \bottomrule
    \end{tabular}%
    }
  \label{tab:overhead tradeoff}%
  \vspace{1mm} 
  \begin{minipage}{0.7\linewidth}
      \footnotesize Note: Training time is measured in seconds.
  \end{minipage}
\end{table}%
\section{Conclusion}
This research investigates the comparative advantages of meta-learning over conventional whole class training algorithms in few-shot classification scenarios. We introduce an entropy-limited supervised setting for fair comparison of meta-learning and whole class training. Our theoretical results show that meta-learning achieves a tighter generalization error bound under this setting. In addition, we reveal its efficient utilization of limited entropy and robustness to label noise and heterogeneous tasks. Building upon these insights, we propose \ourmethod, a meta-learning framework designed to enhance unsupervised performance. \ourmethod~ utilizes the adaptive clustering algorithm DBSCAN with a dynamic head for unsupervised task construction and a stability-based meta-scaler for robustness against label noise. Our experiments further support the theoretical results and demonstrate the superiority of \ourmethod~in few-shot and zero-shot tasks.

\clearpage
\section*{Acknowledgements}
This work was supported by the China Scholarship Council (CSC) under Grant No.~202406160071, the National Key Research and Development Program of China under Grant No.~2023YFB4502701, the National Natural Science Foundation of China under Grant No.~62232007, and the Pioneer Centre for AI, DNRF grant number P1.
{
    \small
    \bibliographystyle{ieeenat_fullname}
    \bibliography{main}
}

\clearpage
\appendix
\section{Proof of the Lemma 1}\label{sec:a_appendix}

\begin{lemma}\label{lemma:correct_num}
Let the sample volume of the dataset be \(m\), the number of classes be \(C\), the sample number per class be balanced, and the entropy consumed by annotation be \(H\). Then, the expectation of correct labeled samples, i.e., $m'$, is given by
\begin{align}
 m' = \frac{m}{C}e^{\frac{H}{m}}~~~~~~~~\textit{s.t.}~~H \in [0 , mlogC].
\end{align}
\end{lemma}

\begin{proof}
Since the numbers of samples per class are balanced, the categories for each sample have \( C \) equal possibilities. The entropy of one sample is calculated by
\begin{align*}
    H_i &= -\sum_{i=1}^{C} \frac{1}{C} \log \frac{1}{C} \\
    &=\log C.
\end{align*}
Given that the samples demonstrate I.I.D., we can determine that the total entropy equals 
$$
H_{sum}=m\log C.
$$
When the reduction of entropy due to annotation is $H$, the total entropy becomes
$$
H_{sum}-H.
$$
Then, after annotation, the number of possible states for each sample can be written as
$$
C'= e^{\frac{H_{sum}-H}{m}}.
$$
Then, the probability of correct labeling is
\begin{align*}
    p=&\frac{1}{e^{\frac{H_{sum}-H}{m}}} \\
    =&\frac{1}{e^{\frac{m\log C-H}{m}}} \\
    =&e^{\frac{H}{m}}/{C}.
\end{align*}
The expectation of correct labeling samples is
\begin{align*}
    m'=\frac{m}{C}e^{\frac{H}{m}}.
\end{align*}

\end{proof}

\section{Proof of Theorem 1, Theorem 2, and Corollary 1}\label{sec:b_appendix}
\begin{definition}[Uniform stability of single-task learning]
A single-task learning algorithm \( A \) has uniform stability \( \beta \) with respect to the loss function \( l \) if the following holds for any training set \( S \) and for all \( i \in \{1, \dots, m\} \), where \( S^{\backslash i} \) denotes the dataset with the \( i \)-th sample removed:
\[
|\hat{L}(A(S), S) - \hat{L}(A(S^{\backslash i}), S)| \leq \beta.
\]
\end{definition}

\begin{definition}[Uniform stability of meta-learning]
A meta-algorithm $\bm{\mathcal{A}}$ has uniform stability \( \tilde{\beta} \) with respect to the loss function \( l \) if the following holds for any meta-sample  $\bm{\mathcal{S}}$ and for all \( i \in \{1, \dots, n\} \), \( D \sim \tau \), \( S^{tr} \sim D^m \):
\[
|\hat{L}(\bm{\mathcal{A}}(\bm{\mathcal{S}})(S^{tr}), S^{tr}) - \hat{L}(\bm{\mathcal{A}}(\bm{\mathcal{S}}^{\backslash i})(S^{tr}), S^{tr})| \leq \tilde{\beta}.
\]
\end{definition}

\begin{theorem}[Generalization error of entropy-limited WCT]\label{theorem:single_gen_limited_entropy}
Let the sample volume of the dataset be \( m \), the number of classes be $C_1$, the annotation entropy be $H$, and the single-task learning algorithm $\bm{A}$ have uniform stability $\beta$. Then the generalization error $R_{gen}(\bm{A})$ is bounded by the following equation with probability at least \( 1 - \delta \) for any \( \delta \in (0,1) \),
\begin{equation}
R_{gen}(\bm{A}) \leq 2\beta + (4m\beta+M)\sqrt{\frac{C_1\ln(1/\delta)}{2me^{H/m}}}.
\end{equation}
\end{theorem}

\begin{proof}
    According to Theorem~\ref{theorem:single_gen}, under the conventional supervised classification setting, the conventional single-task learning algorithm $\bm{A}$, \ie, WCT, has the following generalization error upper bound,
    $$
        R_{gen}{\bm{A}} \leq 2\beta + (4m\beta + M) \sqrt{\frac{\ln(1/\delta)}{2m}}.
    $$
    Under entropy-limited setting, Lemma~\ref{lemma:correct_num} derives the number of correct labeling samples. As indicated in Section 2.2, some algorithms are robust to label noise. However, we consider the worst-case scenario here, \ie, only samples with correct labels are taken into account. As a result, we replace $m$ in the above equation with $m'$ given by Lemma~\ref{lemma:correct_num}, \ie,
    $$
    R_{gen}(\bm{\mathcal{A}}) \leq 2\beta + 2\tilde{\beta}+(4n\tilde{\beta}+M)\sqrt{\frac{kC^2_2\ln(1/\delta)}{2me^{H/m}}}.
    $$
    Note that we don't replace $m$ in $4m\beta$, because it will be asymptotically eliminated by $\beta$ in the derivation of Corollary~\ref{corollary:main}.
\end{proof}

\begin{theorem}[Generalization error of entropy-limited meta-learning]\label{theorem:meta_gen_limited_entropy}
Let the sample volume of the dataset be \( m \), the number of classes per task be $C_2$, the number of samples per class be $k$, the number of tasks be $n$, the annotation entropy be $H$, the base-learner $\bm{A}$ has uniform stability $\beta$, and the meta-learner $\bm{\mathcal{A}}$ have uniform stability $\tilde{\beta}$. Then generalization error $R_{gen}(\bm{\mathcal{A}})$ is bounded by the following equation with probability at least \( 1 - \delta \) for any \( \delta \in (0,1) \),
\begin{equation}\label{eq:meta_generalization_bound}
R_{gen}(\bm{\mathcal{A}}) \leq 2\beta + 2\tilde{\beta}+(4n\tilde{\beta}+M)\sqrt{\frac{kC^2_2\ln(1/\delta)}{2me^{H/m}}}.
\end{equation}
\end{theorem}

\begin{proof}
    According to Theorem~\ref{theorem:single_gen}, under the conventional supervised classification setting, the meta-learning algorithm $\bm{\mathcal{A}}$, has the following generalization error upper bound,
    \[
    R_{gen}(\bm{\mathcal{A}}) \leq  2\tilde{\beta} + (4n\tilde{\beta} + M) \sqrt{\frac{\ln(1/\delta)}{2n}} + 2\beta.
    \]
    The number of tasks $n$, in the worst-case scenario, satisfies
    $$
    n=\frac{m}{kC_2}
    $$
    Similar to the proof of Theorem~\ref{theorem:single_gen_limited_entropy}, replace $m$ in the above equation with $m'$ given by Lemma~\ref{lemma:correct_num}, we have
    $$
    R_{gen}(\bm{\mathcal{A}}) \leq 2\beta + 2\tilde{\beta}+(4n\tilde{\beta}+M)\sqrt{\frac{kC^2_2\ln(1/\delta)}{2me^{H/m}}}.
    $$
    Similar to the proof of Theorem~\ref{theorem:single_gen_limited_entropy}, we don't replace $n$ in $4n\tilde{\beta}$.
\end{proof}

\begin{corollary}\label{corollary:main}
Let the base-level stability $\beta \sim o(\sqrt{1/m})$, the meta-level stability $\tilde{\beta} \sim o(\sqrt{1/n})$, and the entropy resource $H$ be equal for each algorithm. Then, the meta-learning algorithm $\bm{\mathcal{A}}$ has a tighter generalization error upper bound than the single-task learning algorithm $\bm{A}$ when
\begin{equation}\label{eq:theorem_result}
    C_2^2 \cdot k < C_1.
\end{equation}
\end{corollary}
\begin{proof}
    According to Bousquet~\cite{stability_and_generalization} and ~\citet{generalization_stable_2} et.al., the uniform stability $\beta$ and $\tilde{\beta}$ of algorithms decrease as the dataset scale increases, typically satisfying $\beta \sim o(\sqrt{1/m})$ and $\tilde{\beta} \sim o(\sqrt{1/n})$, respectively. To ensure that the generalization error formula holds with high probability, $\delta$ is typically minimized. When $m$ is sufficiently large, these conditions ensure that the generalization error of $\bm{A}$ is dominated by
    $$
    M\sqrt{\frac{C_1\ln(1/\delta)}{2me^{H/m}}},
    $$
    and the generalization error of $\bm{\mathcal{A}}$ is dominated by
    $$
    M\sqrt{\frac{kC^2_2\ln(1/\delta)}{2me^{H/m}}}.
    $$
    In Theorem~\ref{theorem:meta_gen_limited_entropy}, $n$ is significantly underestimated. Therefore, as long as $kC_2^2 < C_1$, or even when $kC_2^2 \sim O(C_1)$, the meta-learning algorithm $\bm{\mathcal{A}}$ admits a much tighter upper bound on the generalization error.
\end{proof}

\begin{theorem}[Generalization error of single-task learning~\cite{generalization_stable_2}]\label{theorem:single_gen}
For any data distribution \( \mathcal{D} \) and training set \( S \) with \( m \) samples, if a single-task learning algorithm \( \bm{A} \) has uniform stability \( \beta \) with respect to a loss function \( l \) bounded by \( M \), then the following statement holds with probability at least \( 1 - \delta \) for any \( \delta \in (0,1) \):

\[
R(\bm{A}, \mathcal{D}) \leq \hat{R}(\bm{A},S) + 2\beta + (4m\beta + M) \sqrt{\frac{\ln(1/\delta)}{2m}}.
\]

\end{theorem}

\begin{theorem}[Generalization error of meta-learning~\cite{generalization_stable_2}]\label{theorem:meta_gen}
For any task distribution \( \tau \) and meta-sample \( \bm{\mathcal{S}} \) with \( n \) tasks, if a meta-algorithm \( \bm{\mathcal{A}} \) has uniform stability \( \tilde{\beta} \) and the inner-task algorithm \( A(S) \) has uniform stability \( \beta \) with respect to a loss function \( l \) bounded by \( M \), then the following statement holds with probability at least \( 1 - \delta \) for any \( \delta \in (0,1) \):

\[
R(\bm{\mathcal{A}}(\bm{\mathcal{S}}), \tau) \leq \hat{R}(\bm{\mathcal{A}}(\bm{\mathcal{S}}), S) +  2\tilde{\beta} + (4n\tilde{\beta} + M) \sqrt{\frac{\ln(1/\delta)}{2n}} + 2\beta.
\]
\end{theorem}

\section{Experimental Details}\label{sec:c_appendix}
\subsection{Dataset setup}\label{sec:dataset setup}
\noindent\textbf{Omniglot.}
The raw dataset contains 1628 classes, we split the classes of training set, evaluation set, test set into 800: 400: 432. We use Omniglot in three scenarios. The first scenario is in Section 2.2. We perform supervised few-shot learning with label noise. We randomly mask the labels of the samples in the training set according to the noise ratio (\ie, 0\%, 15\%, 30\%). Depending on the training method, we can construct these raw data into task followed by~\citet{MAML}, or use them directly for whole class training followed by~\citet{a_good_embedding}. The second scenario is in Section 2.3. We perform supervised few-shot learning with heterogeneous tasks. When constructing heterogeneous tasks, we sample a variable number of classes, to ensure the difference in the way of tasks (\ie, 5-20 way), and further to ensure the heterogeneity. The third scenario is in Section 4.1. We perform unsupervised few-shot learning. We follow the protocol given by~\citet{CACTUs}.
\vspace{5pt}
\par
\noindent\textbf{Mini-Imagenet.}
The raw Mini-Imagenet contains 100 classes, we the split classes of training set, evaluation set, test set into 64: 16: 20. We use Mini-Imagenet in three scenario. The details of the setup of the three experimental scenarios are the same as Omniglot. With the except that we construct 5-10 way heterogeneous task in Section 2.3.
\vspace{5pt}
\par
\noindent\textbf{CIFAR-10, CIFAR-100, STL-10, Imagennet, and Tiny Imagenet.}
For CIFAR-10, CIFAR-100, STL-10, Imagennet, and Tiny Imagenet datasets, we follow the protocol given by~\citet{ReSSL}. They are used for unsupervised zero-shot learning, so we mask all the labels in training set.
\vspace{5pt}
\par
\noindent\textbf{DomainNet.}
DomainNet is a domain adaption dataset. We use it to evaluate algorithms' ability of unsupervised zero-shot domain adaption. It contains 6 domain with 345 classes for each domain. We use one domain for test and the remain 5 domain for both training and validating. Note that when constructing tasks, we sample classes from the same domain and we mask all the labels in training set.
\vspace{5pt}
\par
\noindent\textbf{MobileNet40 and ShapeNetCore.}
ModelNet40 contains 12311 CAD models across 40 categories, primarily used for 3D shape classification and point cloud analysis. ShapeNetCore includes over 51300 3D models spanning 55 categories, serving as a benchmark for 3D classification, segmentation, and reconstruction. For the few-shot classification task, we follow the settings given by~\citet{3D_baseline}.

\subsection{Algorithm Setup}\label{sec:algorithm setup}
\par
\noindent\textbf{MINO.}
We use MINO in both unsupervised zero-shot and few-shot scenario. In unsupervised few-shot datasets, we follow the same backbone architecture given by~\href{https://github.com/dragen1860/MAML-Pytorch}{github.com/dragen1860/MAML-Pytorch}. We set epoch, inner-loop learning rate, outer-loop learning rate, meta-batch size, inner-loop step, and number of sample per task, as 30000, 0.05, 0.001, 8, 5, and 50 respectively. For DBSCAN, we set min\_samples and eps as 15 and 1.0, respectively.
In unsupervised zero-shot datasets (except of DomainNet), we follow the same backbone architecture given by~\href{https://github.com/xu-ji/IIC}{github.com/xu-ji/IIC}, \ie, ResNet and VGG11. We set epoch, inner-loop learning rate, outer-loop learning rate, meta-batch size, adaption steps for evaluation and sub-sample size, as 80000, 0.001, 0.001, 8, 0, and 100 respectively. For DBSCAN , we set min\_samples and eps as 15 and 1.0, respectively. For DomainNet dataset, we use ResNet-9 as backbone architecture, which is the same as~\href{https://github.com/liyunsheng13/DRT}{github.com/liyunsheng13/DRT}. The other configuration is the same as other unsupervised zero-shot datasets. Note that for fair comparison, we keep the network structure of other methods consistent with that of MINO to ensure that the neural networks have the same scale. As for the hyperparameters, we follow the settings provided in the original paper.
\vspace{5pt}
\par
\noindent\textbf{WCT, MAML, ANIL, and MTL.}
In Section 2.2, Section 2.3, and Section 4.3, we use WCT, MAML, ANIL, and MTL to perform experiments, on Omniglot and Mini-Imagenet datasets. For MAML, we leverage the "body" given by~\cite{match_network}, which has 4 convolution modules. Each module consist of a 3 × 3 convolutions and 64 filters, followed by a batch normalization, a ReLU nonlinearity, and 2 × 2 max-pooling. We utilize the "head" followed the baseline classifier given by~\citet{MAML}. For Omniglot, we used strided convolutions instead of max-pooling. For ANIL, its difference with MAML lies solely in the strategy for updating the head~\cite{ANIL}. For WCT, we use the same neural network architecture and learning configuration, except that its head is fixed, and the length of the head equals the number of whole classes. For MTL, we also maintain the same neural network architecture and learning configuration, except that it uses a single-level optimization strategy. In Section 2.3, for SHM, we train a model for each way of tasks, and ultimately take the average testing performance of the models. For DHM and MTL, we train the model with the train set consisting of a mixture of the heterogeneous tasks, and ultimately evaluating its performance directly on the test set.
\vspace{5pt}
\par
\noindent\textbf{PsCo, Meta-GMVAE, UMTRA, and CACTUs.}
We reuse the configuration given by~\citet{PsCo},\citet{Meta-GMVAE},\citet{UMTRA}, and~\citet{CACTUs}, since our test scenarios are the same as theirs.
\vspace{5pt}
\par
\noindent\textbf{ReSSL and IIC.}
In CIFAR-10, CIFAR-100, STL-10, ImageNet, and Tiny ImageNet datasets, we reuse the configuration given by~\citet{ReSSL} and~\citet{IIC}. In DomainNet dataset, for a fair comparison, we use ResNet-9 as backbone and maintain the same learning configuration as mentioned above.
\vspace{5pt}
\par
\noindent\textbf{MAE and NVAE.}
For a fair comparison, we use the same backbone provided by~\citet{base-VAE}, which has a similar number of parameters as other models.
\vspace{5pt}
\par
\noindent\textbf{DeepCluster.}
We run DeepCluster for each unsupervised zero-shot dataset, which we respectively randomly crop and resize to the appropriate image size. We modify the first layer of the AlexNet architecture used by the authors to accommodate this input size. We follow the authors and use the input to the (linear) output layer as the embedding. These are 4096-dimensional, so we follow the authors and apply PCA to reduce the dimensionality to 256, followed by whitening. The configuration is built upon~\href{https://github.com/facebookresearch/ deepcluster}{github.com/facebookresearch/deepcluster}. In DomainNet dataset, we also use ResNet-9 as backbone.
\vspace{5pt}
\par
\noindent\textbf{BiGAN.}
We follow the BiGAN authors and specify a uniform 50-dimensional prior on the unit hypercube for the latent. They use a 200 dimensional version of the same prior for their ImageNet experiments, so we follow suit for our unsupervised zero-shot dataset. Our configuration is built upon \href{https://github.com/jeffdonahue/bigan}{github.com/jeffdonahue/bigan}. In DomainNet dataset, we also use ResNet-9 as backbone.
\section{Discussion and Future work.}\label{sec:future_work}
\noindent\textbf{Equal Classes Probability Assumption.}
We hold a balanced classes probability assumption in Lemma~\ref{lemma:correct_num}, which may raise scrutiny and challenge. Because in the unsupervised setting, the samples included in the constructed tasks are random. However, according the conclusion given by~\citet{UMTRA}, for a task $T_i$, the probability that all samples are in a different class is equal to 
$$
P = \frac{C_1! \cdot k^{C_2} \cdot (C_1 \cdot k - C_2)!}{(C_1 - C_2)! \cdot (C_1 \cdot k)!}.
$$
When $ C_1 \gg C_2$, we have $p \to 1$, which implies that the equal probability assumption holds.
\vspace{5pt}
\par
\noindent\textbf{Experiments Under Equal Entropy Conditions.}
Entropy is not a directly utilized metric in the algorithm training process. Instead, it represents the resources required for annotation. Therefore, we leverage the label noise ration given by Lemma~\ref{lemma:correct_num}, \ie,
\begin{align*}
    p=e^{\frac{H}{m}}/{C_1},
\end{align*}
to infer the expected label noise in the dataset. We then use this noised dataset for training different models to simulate equal entropy conditions.
\vspace{5pt}
\par
\noindent\textbf{The Underestimated Number of Tasks $n$.}
In the worst-case scenario, according to the definition in Theorem~\ref{theorem:meta_gen_limited_entropy}, the number of tasks should be
$$
n=\frac{m}{kC_2}.
$$
However, in a typical unsupervised setting, the samples in a task are independently and randomly sampled. Assuming the total number of task samples is $kC_2$ , the number of different tasks we can construct is
$$
\binom{m}{kC_2} = \frac{m!}{(kC_2)! (m - kC_2)!}.
$$
This number is significantly larger than what we use in the paper. As a result, the actual generalization error bound of meta-learning under limited entropy conditions is much tighter than the one derived in Theorem~\ref{theorem:meta_gen_limited_entropy}. This reinforces the conclusion given by Corollary~\ref{corollary:main}.
\vspace{5pt}
\par

\noindent\textbf{Overhead.}
MINO incurs low computational overhead by replacing existing components rather than adding new ones. Specifically, it replaces CACTUs' K-means with DBSCAN, uses a dynamic head instead of a static one, and introduces a lightweight meta-scaler based on batch-level linear computation. As shown in Table~\ref{tab:ablation_overhead}, each component adds less than 7\% training overhead compared to CACTUs, while achieving a 13.21\% accuracy gain (see Table~\ref{tab:comp_unsup_few_class}). Importantly, all components are used only in training and do not affect inference overhead.
\begin{table}[h]
    \centering
    \vspace{-5pt}
    \caption{Per-sample latency (ms) on CIFAR-100 during training and inference, compared with CACTUs.}
    \vspace{-10pt}
    \resizebox{\linewidth}{!}{
        \begin{tabular}{lll}
            \toprule
            \textbf{Setting} & \textbf{Training Latency (ms)} & \textbf{Inference Latency (ms)} \\
            \midrule
            \textbf{CACTUs}             & 4.98 & 1.56 \\
            \textbf{W/O DBSCAN}         & 5.32~~\textcolor{red}{$\uparrow$~0.34} & 1.54~~\textcolor{green}{$\downarrow$~0.02} \\
            \textbf{W/O Meta-Learning (WCT)} & 4.26~~\textcolor{green}{$\downarrow$~0.72} & 1.57~~\textcolor{red}{$\uparrow$~0.01} \\
            \textbf{W/O Meta-Scaler}    & 5.29~~\textcolor{red}{$\uparrow$~0.31} & 1.56~~\textcolor{red}{$\uparrow$~0.00} \\
            \textbf{MINO (Ours)}        & 5.51~~\textcolor{red}{$\uparrow$~0.53} & 1.57~~\textcolor{red}{$\uparrow$~0.01} \\
            \bottomrule
        \end{tabular}
    }
    \label{tab:ablation_overhead}
\end{table}
\vspace{5pt}
\par

\noindent\textbf{Future Work.}
We will establish a more rigorous theoretical framework. By improving the computation method for the entropy of samples, this framework will no longer rely on the equal classes probability assumption, enabling it to cover datasets with arbitrary class probability distributions. At the same time, based on the latest research~\cite{noise_claasify, noise_federal}, we need to incorporate the impact of label noise on the generalization ability of the algorithm into this theoretical framework. This allows us to go beyond considering only the worst-case scenario in terms of sample size, as in Proof of Theorem~\ref{theorem:single_gen_limited_entropy}, and instead provide a tighter generalization error upper bound.

\section{Related Work}
\noindent\textbf{Theoretical analysis of meta-learning.}
In recent years, significant advances have been made in meta-learning theory research. Studies have revealed important insights about MAML's fast adaptation mechanisms~\cite{ANIL,unreval}, while other research has demonstrated that well-designed embeddings can potentially outperform meta-learning approaches in few-shot classification tasks~\cite{a_good_embedding}. Further investigations have established the independence between meta-training and adaptation algorithms~\cite{a_closer_again}, and identified specific conditions where baseline methods can exceed meta-learning performance in few-shot classification tasks~\cite{a_closer_look}. Additional research has examined the implications of class and novel class generalization in meta-learning~\cite{meta_baseline}, while theoretical work has established upper bounds for meta-learning generalization error~\cite{generalization_PAC_1,generalization_information_2,generalization_stable_2,generalization_PAC_2,generalization_PAC_3,genralization_information_1,generalization_stable_1}. This paper addresses two fundamental questions regarding meta-learning: Under what circumstances does meta-learning demonstrate superior performance compared to alternative algorithms in few-shot classification, and what underlying factors contribute to this advantage?
\\
\par
\noindent\textbf{Unsupervised meta-learning.}
Unsupervised meta-learning~\cite{CACTUs,UMTRA,Meta-GMVAE,PsCo,HMS,UML1} links meta-learning and unsupervised learning by constructing synthetic tasks and extracting the meaningful information from unlabeled data. For example, CACTUs~\cite{CACTUs} cluster the data on the pretrained representations at the beginning of meta-learning to assign pseudo-labels. Instead of pseudo-labeling, UMTRA~\cite{UMTRA} and LASIUM~\cite{LASIUM} generate synthetic samples using data augmentations or pretrained generative networks like BigBiGAN~\cite{BiGAN}. Meta-GMVAE~\cite{Meta-GMVAE} and Meta-SVEBM~\cite{Meta-SVEMB} represent unknown labels via categorical latent variables using variational autoencoders~\cite{NVAE} and energy-based models, respectively. In this paper, we leverage the sights under entropy-limited supervised setting, improve meta-learning algorithm's robustness against label noise and heterogeneous tasks.
  
\end{document}